%% file: paper.tex
\documentclass{article}

\usepackage{subfigure}

\usepackage[utf8]{inputenc} 
\usepackage[T1]{fontenc}    
\usepackage[hidelinks]{hyperref}       
\usepackage{url}            
\usepackage{booktabs}       
\usepackage{amsfonts}       
\usepackage{nicefrac}       
\usepackage{microtype}      

\usepackage{amsmath}
\usepackage{amsthm}
\usepackage{dsfont}
\usepackage{graphicx}
\usepackage{multirow}
\usepackage{xcolor} 
\usepackage{thmtools} 
\usepackage{thm-restate}
\usepackage[toc,page]{appendix}
\usepackage{wrapfig}

\usepackage{balance} 

\DeclareMathOperator*{\argmax}{arg\,max}

\newtheorem{theorem}{Theorem}

\newtheorem{assumption}{Assumption}

\newcommand\Tstrut{\rule{0pt}{2.6ex}}       
\newcommand\Bstrut{\rule[-0.9ex]{0pt}{0pt}} 
\newcommand{\TBstrut}{\Tstrut\Bstrut} 

\newcommand\numberthis{\addtocounter{equation}{1}\tag{\theequation}}

\usepackage[accepted]{icml2020}

\newcommand{\papertitle}{Educating Text Autoencoders: Latent Representation Guidance via Denoising}

\begin{document}

\twocolumn[
\icmltitle{\papertitle}



\icmlsetsymbol{equal}{*}

\begin{icmlauthorlist}
\icmlauthor{Tianxiao Shen}{mit}
\icmlauthor{Jonas Mueller}{amazon}
\icmlauthor{Regina Barzilay}{mit}
\icmlauthor{Tommi Jaakkola}{mit}
\end{icmlauthorlist}

\icmlaffiliation{mit}{MIT CSAIL}
\icmlaffiliation{amazon}{Amazon Web Services}

\icmlcorrespondingauthor{Tianxiao Shen}{tianxiao@mit.edu}

\icmlkeywords{Machine Learning, ICML}

\vskip 0.3in
]



\printAffiliationsAndNotice{}  

\begin{abstract}

Generative autoencoders offer a promising approach for controllable text generation by leveraging their latent sentence representations.
However, current models struggle to maintain coherent latent spaces required to
perform meaningful text manipulations via latent vector operations.
Specifically, we demonstrate by example that neural encoders do not necessarily map similar sentences to nearby latent vectors. A theoretical explanation for this phenomenon establishes that high-capacity autoencoders can learn an arbitrary mapping between sequences and associated latent representations.
To remedy this issue, we augment adversarial autoencoders with a denoising objective where original sentences are reconstructed from perturbed  versions (referred to as DAAE).
We prove that this simple modification guides the latent space geometry of the resulting model by encouraging the encoder to map similar texts to similar latent representations.
In empirical comparisons with various types of autoencoders, our model provides the best trade-off between generation quality and reconstruction capacity.
Moreover, the improved geometry of the DAAE latent space enables \emph{zero-shot} text style transfer via simple latent vector arithmetic.\footnote{Our code and data are available at \url{https://github.com/shentianxiao/text-autoencoders}}

\end{abstract}

\input{intro.tex}

\input{related.tex}

\input{method.tex}

\input{theory.tex}

\input{experiments.tex}

\input{conclusion.tex}

\section*{Acknowledgements}
We thank all reviewers and the MIT NLP group for their thoughtful feedback.

\clearpage

\bibliography{paper}
\bibliographystyle{icml2020}

\newpage
\clearpage
\input{appendix.tex}

\end{document}

%% file: intro.tex
\section{Introduction}

\begin{figure}[t]
\centering
\includegraphics[width=0.48\textwidth]{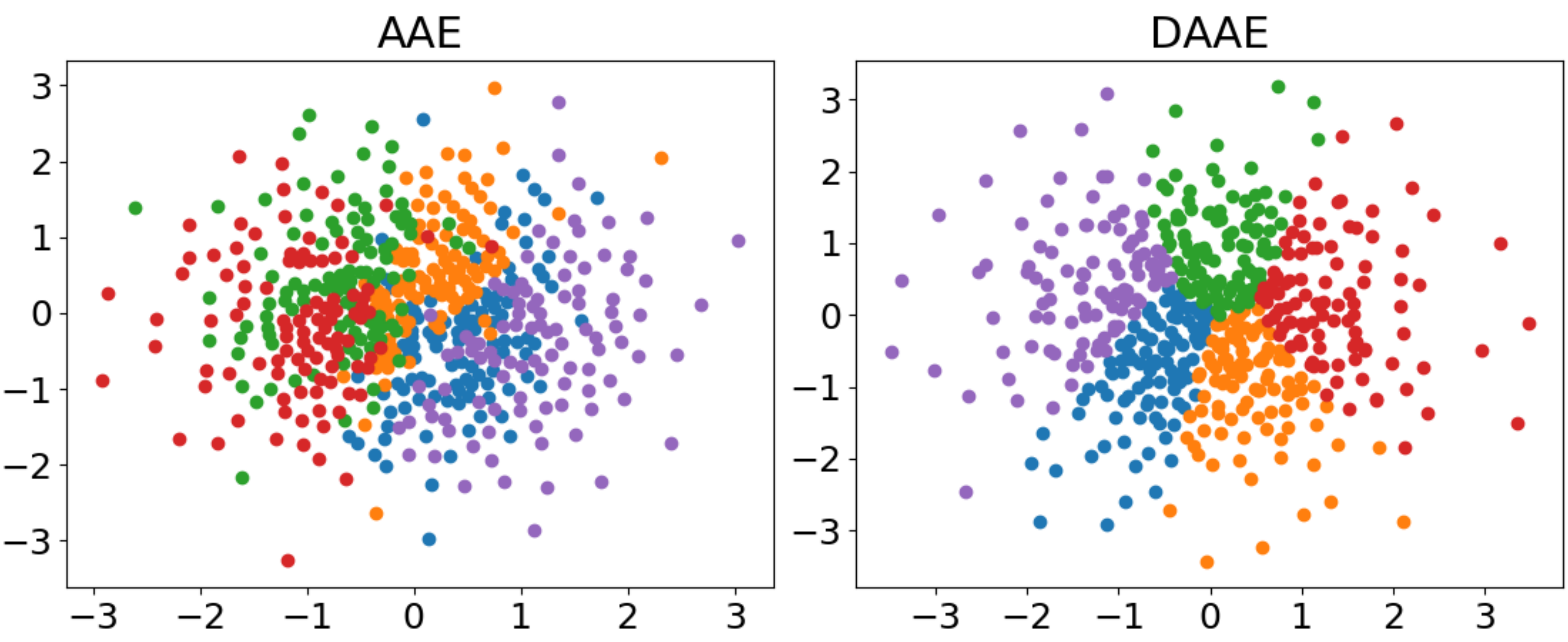}
\caption{\small 
Latent representations learned by AAE and DAAE when mapping clustered sequences in 
$\mathcal X=\{0, 1\}^{50}$ to $\mathcal Z=\mathbb R^2$. The training data stem from 5 underlying clusters, with 100 sequences sampled from each (colored accordingly by cluster identity).
} 
\label{fig:toy_exp}
\end{figure}

Autoencoder-based generative models have become popular tools for advancing controllable text generation such as style or sentiment transfer~\citep{bowman2015generating,hu2017toward,shen2017style,zhao2018adversarially}. By representing sentences as vectors in a latent space, these models offer 
an attractive continuous approach to manipulating text by means of simple latent vector arithmetic. However, the success of such manipulations rests heavily on the geometry of the latent space representations and its ability to capture underlying sentence semantics. We discover that without additional guidance, fortuitous geometric 
alignments are unlikely to arise, shedding light on challenges faced by existing methods.

In this work, we focus on the latent space geometry of adversarial autoencoders \citep[AAEs]{makhzani2015adversarial}. 
In contrast to variational autoencoders~\citep[VAEs]{kingma2013auto}, AAEs maintain a strong coupling between their encoder and decoder, ensuring that the decoder does not ignore sentence representations produced by the encoder~\citep{bowman2015generating}. 
The training objective for AAEs consists of two parts: the ability to reconstruct sentences and an additional constraint that the sentence encodings follow a given prior distribution (typically Gaussian). 
We find that these objectives alone do not suffice to enforce proper latent space geometry: 
in a toy example with clustered data sequences, a perfectly-trained AAE undesirably mixes different clusters in its latent space (Figure~\ref{fig:toy_exp}, Left).

We provide a theoretical explanation for this phenomenon by analyzing high-capacity encoder/decoder networks in modern sequence models.
For discrete objects such as sentences where continuity assumptions no longer hold, 
powerful AAEs can 
learn to map training sentences to latent prior samples arbitrarily, while 
retaining perfect reconstruction. In such cases, even minimal latent space manipulations can yield random, unpredictable changes in the resulting text.

To remedy this issue, we augment AAEs with a simple denoising
objective~\citep{vincent2008extracting,creswell2018denoising}, requiring perturbed sentences (with random words missing) to be reconstructed back to their original versions. 
We prove that disorganized encoder-decoder mappings are suboptimal under the denoising criterion.
As a result, the denoising AAE model (or DAAE for short) will map similar sentences to similar latent representations.
Empirical studies confirm that denoising promotes sequence neighborhood preservation, consistent with our theory (Figure \ref{fig:toy_exp}, Right).
Our systematic evaluations demonstrate that DAAE maintains the best trade-off between producing high-quality text vs.\ informative sentence representations.
We further investigate the extent to which text can be manipulated via simple transformations of latent representations. DAAE is able to perform sentence-level vector arithmetic~\citep{mikolov2013linguistic} to change the tense or sentiment of a sentence without any supervision during training.
Denoising also helps produce higher quality sentence interpolations, suggesting better linguistic continuity
in its latent space. 

%% file: related.tex
\section{Related Work}

\paragraph{Denoising}
\citet{vincent2008extracting} first introduced denoising autoencoders (DAEs) to learn robust image representations, and
\citet{creswell2018denoising} applied DAAEs to generative image modeling.
Previous analysis of denoising focused on continuous image data and single-layer networks~\citep{poole2014analyzing}.
Here, we demonstrate that input perturbations are particularly useful for discrete text modeling with powerful sequence networks, as they encourage preservation of data structure in latent space representations.



\paragraph{Variational Autoencoder (VAE)}
Apart from AAE that this paper focuses on, another popular latent variable generative model is VAE~\citep {kingma2013auto}. Unfortunately, when the decoder is a powerful autoregressive model (such as a language model),
VAE suffers from the \emph{posterior collapse} problem where the latent representations are ignored 
\citep{bowman2015generating,chen2016variational}.
If denoising is used in conjunction with VAE~\citep{im2017denoising} in text applications, then the noisy inputs will only exacerbate VAE's neglect of the latent variable.
\citet{bowman2015generating} proposed to dropout words on the decoder side to alleviate VAE’s collapse issue. However, even with a weakened decoder and other techniques including KL-weight annealing and adjusting training dynamics, it is still difficult to inject significant content into the
latent code~\citep{yang2017improved,kim2018semi,he2019lagging}. Alternatives like the $\beta$-VAE~\citep{higgins2017beta} appear necessary.

\paragraph{Controllable Text Generation}
Previous work has employed autoencoders trained with attribute label information to control text generation~\citep{hu2017toward,shen2017style,logeswaran2018content,subramanian2018multiple}.
We show that the proposed DAAE can perform text manipulations despite being trained in a completely unsupervised manner without any labels.
This suggests that on the one hand, our model can be adapted to semi-supervised learning when a few labels are available. On the other hand, it can be easily scaled up to train one large model on unlabeled corpora and then applied for transferring various styles.

%% file: method.tex
\section{Methods}

Define $\mathcal X=\mathcal V^m$ as the sentence space of  sequences of discrete symbols from vocabulary $\mathcal V$ (with  length $\le m$), and let $\mathcal Z=\mathbb R^d$ denote a continuous space of latent representations.
Our goal is to learn a mapping between a data distribution $p_{\text{data}}(x)$ over $\mathcal X$ and a given prior distribution $p(z)$ over latent space $\mathcal Z$. 
Such a mapping allows us to manipulate discrete data through its continuous latent representation $z$, and provides 
a generative model whereby new data can be sampled by drawing $z$ from the prior and then
mapping it to the corresponding sequence in $\mathcal X$.

\paragraph{Adversarial Autoencoder (AAE)}
The AAE involves a deterministic encoder $E:\mathcal X \rightarrow \mathcal Z$
mapping from data space to latent space,
a probabilistic decoder $G:\mathcal Z \rightarrow \mathcal X$
that generates sequences from latent representations, 
and a discriminator $D: \mathcal Z \rightarrow [0,1]$
that tries to distinguish between encodings of data $E(x)$ and samples from $p(z)$.  
Both $E$ and $G$ are recurrent neural nets (RNNs).\footnote{We also tried Transformer models~\citep{vaswani2017attention}, but they did not outperform LSTMs on our moderate-size datasets.}
$E$ takes input sequence $x$ and uses the final RNN hidden state as its encoding $z$. $G$ generates a sequence $x$
autoregressively, with each step conditioned on $z$ and symbols emitted in preceding steps.
$D$ is a feed-forward net that infers the probability of $z$ coming from the prior rather than the encoder.
$E$, $G$ and $D$ are trained jointly with a min-max objective:
\begin{align}
    \min_{E,G}\max_{D}  ~ \mathcal L_{\text{rec}}(\theta_E,\theta_G) - \lambda \mathcal L_{\text{adv}}(\theta_E,\theta_D)\label{eq:aae}
\end{align}
with:
\begin{align}
  \ \    \mathcal L_{\text{rec}}(\theta_E,\theta_G) =~ &  \mathbb E_{p_{\text{data}}(x)}[-\log p_G(x|E(x))] 
    \label{eq:rec}
    \\
    \mathcal L_{\text{adv}}(\theta_E,\theta_D) =~ &  \mathbb E_{p(z)}[-\log D(z)] ~+ \nonumber \\
    & \mathbb E_{p_{\text{data}}(x)}[-\log(1-D(E(x)))]
\end{align}
where reconstruction loss $\mathcal L_{\text{rec}}$ and adversarial loss\footnote{We actually train $E$ to maximize $\log D(E(x))$ instead of $-\log(1-D(E(x)))$, which is more stable in practice~\citep{goodfellow2014generative}. We also tried the alternative WGAN objective~\citep{arjovsky2017wasserstein} but did not notice any gains.} $\mathcal L_{\text{adv}}$ are  weighted via hyperparameter $\lambda>0$ during training.

\paragraph{Denoising Adversarial Autoencoder (DAAE)}
We extend the AAE by introducing local $x$-perturbations and requiring reconstruction of each original $x$ from a randomly perturbed version.  
As we shall see, this implicitly encourages similar sequences to  map to similar latent representations, without requiring any additional training objectives. 
Specifically, given a perturbation process $C$ that stochastically corrupts $x$ to some nearby ${\tilde x \in \mathcal{X}}$, let $p(x,\tilde x)=p_{\text{data}}(x)p_C(\tilde x|x)$ and $p(\tilde x)=\sum_x p(x,\tilde x)$. The corresponding DAAE training objectives are:
\begin{align}
  ~~~~~~~~ \mathcal L_{\text{rec}}(\theta_E,\theta_G) =~ &  \mathbb E_{p(x,\tilde x)}[-\log p_G(x|E(\tilde x))] \label{eq:rec_denoise}
    \\
    \mathcal L_{\text{adv}}(\theta_E,\theta_D) =~ &  \mathbb E_{p(z)}[-\log D(z)] ~+ \nonumber \\
    & \mathbb E_{p(\tilde x)}[-\log(1-D(E(\tilde x)))]
\end{align}
Here, $\mathcal L_{\text{rec}}$ is the loss of reconstructing $x$ from $\tilde x$, and $\mathcal L_{\text{adv}}$ is the adversarial loss
evaluated using perturbed $\tilde x$.
This objective simply combines the denoising technique with AAE~\citep{vincent2008extracting,creswell2018denoising}, resulting in the denoising AAE (DAAE) model.

Let $p_E(z|x)$ denote the encoder distribution (for a deterministic encoder it is concentrated at a single point).
With perturbation process $C$, the posterior distributions of the latent representations are of the form: 
\begin{align}\label{eq:perturbedposterior}
    q(z|x)=\sum_{\tilde x} p_C(\tilde x|x)p_E(z|\tilde x)
\end{align}
This enables the DAAE to utilize stochastic encodings even by merely employing a deterministic 
encoder network trained without any reparameterization-style tricks.
Note that since $q(z|x)$ of the form (\ref{eq:perturbedposterior}) is a subset of all possible conditional distributions,
our model is still minimizing an upper bound of the Wasserstein distance between data and model distributions, as previously shown by~\citet{tolstikhin2017wasserstein} for AAE (see Appendix~\ref{sec:wasserstein} for a full proof).

%% file: theory.tex
\section{Latent Space Geometry}\label{sec:geometry}
\begin{figure*}[t]
\centering
\includegraphics[width=\textwidth]{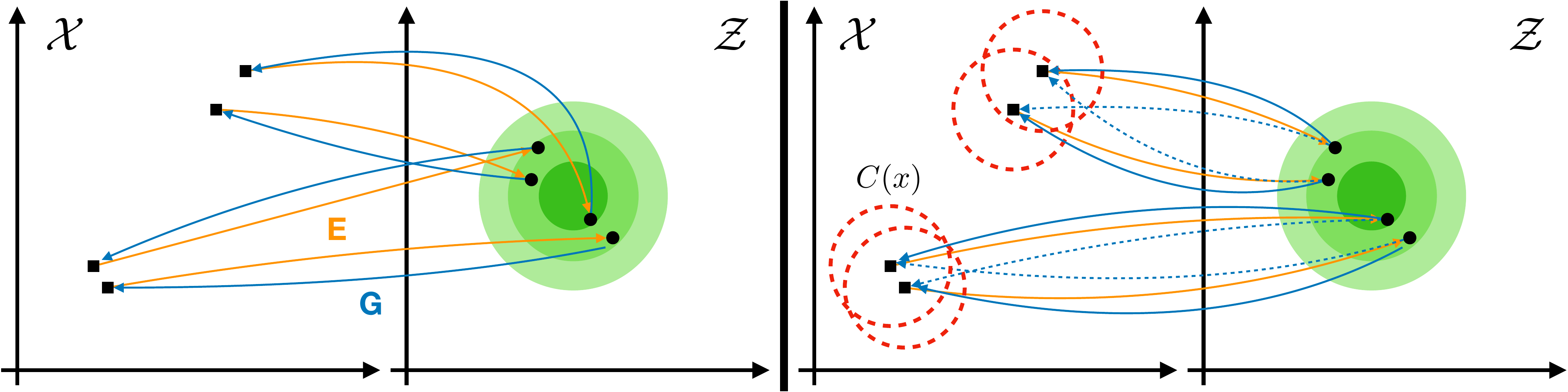}
\caption{\small 
Illustration of the learned latent geometry by AAE before and after introducing $x$ perturbations. With high-capacity encoder/decoder networks, a standard AAE has no preference over $x$-$z$ couplings and thus can learn a random mapping between them (Left). Trained with local perturbations $C(x)$, DAAE learns to map similar $x$ to close $z$ to best achieve the denoising objective (Right).
}\label{fig:illustration}
\end{figure*}

Denoising was previously viewed as a technique to help learn data manifolds and extract more robust representations~\citep{vincent2008extracting,bengio2013generalized}, but there has been little formal analysis of precisely how it helps.
Here, we show that denoising guides the latent space geometry of text autoencoders to preserve neighborhood structure in data.
By mapping similar text to similar representations, we can perform smoother sentence interpolation and can better implement meaningful text manipulations through latent vector operations.
\subsection{A Toy Example}\label{sec:toy_experiment}
We pose the following question: In practice, will autoencoders learn a smooth and regular latent space geometry that reflects the underlying structure of their training data?
To study this, we conduct experiments using synthetic data with a clear cluster structure to see if the clusters are reflected in the learned latent representations.  

We randomly sample 5 binary (0/1) sequences of length 50 to serve as cluster centers. From each cluster, 100 sequences are sampled by randomly flipping elements of the center sequence with a probability of 0.2.  
The resulting dataset has 500 sequences from 5 underlying clusters, where sequences stemming from the same cluster typically have many more elements in common than those from different clusters. 
We train AAE and DAAE models with a latent dimension of 2, 
so the learned representations can be drawn directly. 
Similar results were found using a higher latent dimension and visualizing the representations with t-SNE (see Appendix~\ref{sec:toy_experiment_dim5}).

In terms of its training objectives, the AAE appears very strong, achieving perfect reconstruction on all data points while keeping the adversarial loss around the maximum $-2\log 0.5$ (when $D$  always outputs probability 0.5). 
However, the left panel of Figure~\ref{fig:toy_exp} reveals that, although they are well separated in the data space, different clusters become mixed together in the learned latent space.
This is because that for discrete objects like sequences, neural networks have the capacity to map similar data to distant latent representations.
With only the autoencoding and latent prior constraints, the AAE fails to learn proper latent space geometry that preserves the cluster structure of data.

We now train our DAAE model using the same architecture as 
the AAE, with perturbations $C$ that randomly flip each element of $x$ with probability $p=0.2$.\footnote{We observed similar results for $p=0.1, 0.2, 0.3$.}
The DAAE can also keep the adversarial loss close to its maximum, and can perfectly reconstruct the data at test time when $C$ is disabled, indicating that training is not hampered by our perturbations. Moreover, the DAAE latent space closely captures the underlying cluster structure in the data, as depicted in the right panel of Figure~\ref{fig:toy_exp}.
By simply introducing local perturbations of inputs, we are able to  ensure similar sequences have similar representations in the trained autoencoder.



\subsection{Theoretical Analysis}\label{sec:theory}

In this section, we provide theoretical explanations for our previous findings.
We formally analyze which type of $x$-$z$ mappings will be learned by AAE and DAAE, respectively, to achieve global optimality of their training objectives. 
We show that a well-trained DAAE is guaranteed to learn neighborhood-preserving latent representations, whereas even a perfectly-trained AAE model may learn latent representations whose geometry fails to reflect similarity in the $\mathcal X$ space
(all proofs are relegated to the appendix).


We study high-capacity encoder/decoder networks with a large number of parameters, as is the case for modern sequence models~\citep{schafer2006recurrent, devlin2018bert,radford2019language}.
Throughout, we assume that:

\begin{assumption}
The encoder $E$ is unconstrained and capable of producing any mapping from $x$'s to $z$'s.
\end{assumption}
\begin{assumption}\label{asm:Lipschitz}
The decoder $G$ can approximate arbitrary $p(x|z)$ so long as it remains sufficiently Lipschitz continuous in $z$. Namely, there exists $L > 0$ such that all decoders $G$ obtainable via training satisfy: $\forall x \in \mathcal{X}, \forall z_1, z_2 \in \mathcal{Z}, \ |\log p_G(x|z_1)-\log p_G(x|z_2)| \le L \|z_1-z_2\|$. 
(We denote this set of decoders by $\mathcal G_L$.)
\end{assumption}

The latter assumption that $G$ is Lipschitz in its continuous input $z$ follows prior analysis of language decoders~\citep{mueller2017sequence}. 
When $G$ is implemented as a RNN or Transformer language model, $\log p_G(x|z)$ will remain Lipschitz in $z$ if the recurrent or attention weight matrices have bounded norm, which is naturally encouraged by regularization arising from explicit $\ell_2$ penalties and implicit effects of SGD training~\citep{zhang2017}.
Note we have not assumed $E$ or $G$ is Lipschitz in $x$, which would be unreasonable since $x$ represents discrete text, and when a few symbols change, the decoder likelihood for the entire sequence can vary drastically (e.g.,\ $G$ may assign a much higher probability to a grammatical sentence than an ungrammatical one that only differs by one word). 

We further assume an effectively trained discriminator that succeeds in its adversarial task:

\begin{assumption}\label{asm:discriminator}
The discriminator $D$ ensures that the latent encodings $z_1,\cdots,z_n$ of training examples $x_1,\cdots,x_n$
are indistinguishable from prior samples $z\sim p(z)$. 
\end{assumption}

In all the experiments we have done, training is very stable and the adversarial loss remains around $-2\log 0.5$, indicating that our assumption holds empirically.
Under Assumption~\ref{asm:discriminator}, we can directly suppose that $z_1,\cdots,z_n$ are actual samples from $p(z)$ which are fixed a priori. Here, the task of the encoder $E$ is to map the given training examples to the given latent points, and the goal of the decoder $p_G(\cdot|\cdot)$ is to maximize $-\mathcal L_{\text{rec}}$ under the encoder mapping.
The question now is which one-to-one mapping between $x$'s and $z$'s an optimal encoder/decoder will learn under the AAE objective (Eq.~\ref{eq:rec}) and DAAE objective (Eq.~\ref{eq:rec_denoise}), respectively.


\begin{restatable}{theorem}{badaae}\label{thm:badaae}
For any one-to-one encoder mapping $E$ from $\{x_1,\cdots,x_n\}$ to $\{z_1,\cdots,z_n\}$, the optimal value of objective $\max_{G\in\mathcal G_L} \frac{1}{n}\sum_{i=1}^n \log p_G(x_i|E(x_i))$ is the same.
\end{restatable}

Intuitively, this result stems from the fact that the model receives no information about the structure of $x$, and $x_1,\cdots,x_n$ are simply provided as different symbols.
Hence AAE offers no preference over $x$-$z$ couplings,
and a random matching in which the $z$ do not reflect any data structure is equally good as any other matching (Figure~\ref{fig:illustration}, Left).
Latent point assignments start to differentiate, however, once we introduce local input perturbations. 

To elucidate how perturbations affect latent space geometry, it helps to first consider a simple setting with only four examples $x_1, x_2, x_3, x_4 \in \mathcal{X}$. Again, we consider given latent points $z_1, z_2, z_3, z_4$ sampled from $p(z)$, 
and the encoder/decoder are tasked with learning which $x$ to match with which $z$.
As depicted in Figure~\ref{fig:illustration}, suppose there are two pairs of $x$ closer together and also two pairs of $z$ closer together. 
Let $\sigma$ denote the sigmoid function, we have the following conclusion:

\begin{restatable}{theorem}{fourpoints}\label{thm:fourpoints}
Let $d$ be a distance metric over $\mathcal{X}$. Suppose $x_1,x_2,x_3,x_4$ satisfy that with some $\epsilon > 0$: $d(x_1, x_2) < \epsilon$, $d(x_3, x_4) < \epsilon$, and $d(x_i, x_j) > \epsilon$ for all other $(x_i, x_j)$ pairs.
In addition, $z_1,z_2,z_3,z_4$ satisfy that with some $0<\delta<\zeta$:
$\|z_1 -  z_2\| < \delta$, $\|z_3 - z_4\| < \delta$, and $\|z_i -  z_j \| > \zeta$ for all other $(z_i,z_j)$ pairs.
Suppose our perturbation process $C$ reflects local $\mathcal{X}$ geometry with: 
$p_C(x_i|x_j)=1/2$ if $d(x_i, x_j) < \epsilon$ and $=0$ otherwise. 
For
$ \delta < 1/L\cdot \left( 
2 \log \left( 
\sigma(L \zeta) 
\right) 
+ \log 2
\right) $
and  
$ \zeta > 1/L\cdot \log\left(1/(\sqrt{2} - 1)\right)$,
the denoising objective $\max_{G\in\mathcal G_L} \frac{1}{n}\sum_{i=1}^n\sum_{j=1}^n p_C(x_j|x_i) \log p_G(x_i|E(x_j))$ (where $n=4$) achieves the largest value when encoder $E$ maps close pairs of $x$ to close pairs of $z$.
\end{restatable}

This entails that DAAE will always prefer to map similar $x$ to similar $z$.
Note that Theorem \ref{thm:badaae} still applies here, and AAE will not prefer any particular $x$-$z$ pairing over the other possibilities. 
We next generalize beyond the basic four-points scenario to consider $n$ examples of $x$ that are clustered, and ask whether this cluster organization will be reflected in the latent space of DAAE. 


\begin{restatable}{theorem}{clusteredx}\label{thm:clusteredx}
Suppose $x_1,\cdots,x_n$ are divided into $n/K$ clusters of equal size $K$, with $S_i$ denoting the cluster index of $x_i$. Let the perturbation process $C$ be uniform within clusters, i.e. $p_C(x_i|x_j)=1/K$ if $S_i=S_j$ and $= 0$ otherwise.
With a one-to-one encoder mapping $E$ from $\{x_1,\cdots,x_n\}$ to $\{z_1,\cdots,z_n\}$,
the denoising objective $\max_{G\in\mathcal G_L} \frac{1}{n}\sum_{i=1}^n\sum_{j=1}^n p_C(x_j|x_i) \log p_G(x_i|E(x_j))$ has an upper bound: $\frac{1}{n^2} \sum_{i,j: S_i\ne S_j} \log \sigma(L \|E(x_i)-E(x_j)\|)-\log K$.
\end{restatable}

Theorem~\ref{thm:clusteredx} 
provides an upper bound of the DAAE objective that can be achieved by a particular $x$-$z$ mapping.
This achievable limit is substantially better when examples in the same cluster are mapped to the latent space in a manner that is well-separated from encodings of other clusters. In other words, by preserving input space cluster structure in the latent space, DAAE can achieve better objective values and thus is incentivized to learn such encoder/decoder mappings. 
An analogous corollary can be shown for the case when examples $x$ are perturbed to yield additional inputs $\tilde x$ not present in the training data.
In this case, the model would aim to collectively map each example and its perturbations to a compact group of $z$ points well-separated from other groups in the latent space.


Our synthetic experiments in Section~\ref{sec:toy_experiment} confirm that DAAE maintains the cluster structure of sequence data in its latent space. While these are simulated data, we note natural language often exhibits cluster structure based on topics/authorship but also contains far richer syntactic and semantic structures. In the next section, we empirically study the performance of DAAE on real text data.

%% file: experiments.tex
\section{Experiments}

We test our proposed model and other text autoencoders on two benchmark datasets: \emph{Yelp reviews} and \emph{Yahoo answers}~\citep{shen2017style,yang2017improved}.
We analyze their latent space geometry, generation and reconstruction capacities, and applications to controllable text generation. 
All models are implemented using the same architecture.
Hyperparameters are set to values that produce the best overall generative models (see Section~\ref{sec:gen_rec_tradeoff}).
Detailed descriptions of training settings, human evaluations, and additional results/examples can be found in appendix.

\paragraph{Datasets}
The Yelp dataset is from~\citet{shen2017style}, which has 444K/63K/127K sentences of less than 16 words in length as train/dev/test sets, with a vocabulary of 10K. It was originally divided into positive and negative sentences for style transfer between them.
Here we discard the sentiment label and let the model learn from all sentences indiscriminately.
Our second dataset of Yahoo answers is from~\citet{yang2017improved}. It was originally document-level.
We perform sentence segmentation and keep sentences with length from 2 to 50 words. The resulting dataset has 495K/49K/50K sentences for train/dev/test sets, with vocabulary size 20K.

\paragraph{Perturbation Process} We randomly delete each word with probability $p$, so that perturbations of sentences with more words in common will have a larger overlap.
We also tried replacing each word with a <mask> token or a random word from the vocabulary.
We found that all variants have similar generation-reconstruction trade-off curves; in terms of neighborhood preservation, they are all better than other autoencoders, but word deletion has the highest recall rate. This may be because word replacement cannot perturb sentences of different lengths to each other even if they are similar.
Defining sentence similarity and meaningful perturbations are task specific.
Here, we demonstrate that even the simplest word deletions can bring significant improvements.
We leave it to future work to explore more sophisticated text perturbations.

\paragraph{Baselines} We compare our proposed DAAE with four alternative text autoencoders:
AAE~\citep{makhzani2015adversarial}, latent-noising AAE~\citep[LAAE]{rubenstein2018latent},
adversarially regularized autoencoder~\citep[ARAE]{zhao2018adversarially},
and $\beta$-VAE~\citep{higgins2017beta}.
Similar to our model, the LAAE uses Gaussian perturbations in the latent space (rather than perturbations in the sentence space) to improve AAE’s latent geometry.
However, it requires enforcing an $L_1$ penalty (${\lambda_1\cdot \|\log \sigma^2(x)\|_1}$) on the latent perturbations' log-variance to prevent them from vanishing.  
In contrast, input perturbations in DAAE enable  stochastic latent representations  without parametric restrictions like Gaussianity.


\begin{table*}[t]
    \def\arraystretch{1.1}\setlength{\tabcolsep}{3pt}
    \fontsize{8.3}{10}\selectfont
    \centering
    \begin{tabular}{lll}
            \toprule
            & AAE & DAAE\\
            \midrule
            \textbf{Source} & \textbf{my waitress katie was fantastic , attentive and personable .} & \textbf{my waitress katie was fantastic , attentive and personable .}\\
            & my cashier did not smile , barely said hello . & the manager , linda , was very very attentive and personable .\\
            & the service is fantastic , the food is great . & stylist brenda was very friendly , attentive and professional .\\
            & the employees are extremely nice and helpful . & the manager was also super nice and personable .\\
            & our server kaitlyn was also very attentive and pleasant . & my server alicia was so sweet and attentive . \\
            & the crab po boy was also bland and forgettable . & our waitress ms. taylor was amazing and very knowledgeable .\\
            \midrule
            \textbf{Source} & \textbf{i have been known to eat two meals a day here .} & \textbf{i have been known to eat two meals a day here .}\\
             & i have eaten here for \_num\_ years and never had a bad meal ever . & you can seriously eat one meal a day here .\\
             & i love this joint . & i was really pleased with our experience here .\\
             & i have no desire to ever have it again . & ive been coming here for years and always have a good experience .\\
             & you do n't need to have every possible dish on the menu . & i have gone to this place for happy hour for years .\\
             & i love this arena . & we had \_num\_ ayce dinner buffets for \_num\_ on a tuesday night .\\
           \bottomrule
    \end{tabular}
    \caption{\small Examples of 5 nearest neighbors in the latent Euclidean space of AAE and DAAE on the Yelp dataset.}
    \label{tab:examples_nn_yelp}
\end{table*}

\begin{figure}[t]
  \begin{center}
    \includegraphics[width=0.4\textwidth]{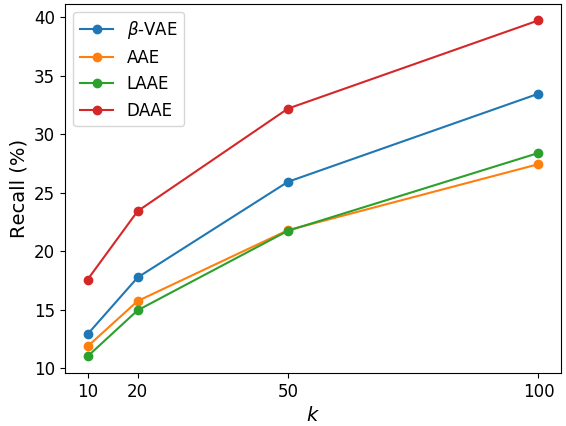}
  \end{center}
  \caption{\small 
   Recall rate of different autoencoders on the Yelp dataset.
   Quantifying how well the latent geometry preserves text similarity, recall is defined as the fraction of each sentence's 10 nearest neighbors in terms of normalized edit distance whose representations lie among the $k$ nearest neighbors in the latent space.
   ARAE has a poor recall $< 1\%$ and thus not shown in the plot.
  }
  \label{fig:recall_yelp}
\end{figure}

\subsection{Neighborhood Preservation}\label{sec:neighborhood}

We begin by investigating whether input perturbations will induce latent space organization that better preserves neighborhood structure in the sentence space.
Under our word-dropout perturbation process, sentences with more words in common are more likely to be perturbed into one another. This choice of $C$ approximately encodes sentence similarity via normalized edit distance.\footnote{Normalized edit distance $\in [0,1]$ is the Levenshtein distance divided by  the max length of two sentences.}
Thus, within the test set, we find both the 10 nearest neighbors of each sentence based on the normalized edit distance (denote this set by $\text{NN}_{x}$), as well as the $k$ nearest neighbors based on Euclidean distance between latent representations (denote this set by $\text{NN}_{z}$). 
We compute the recall rate $|\text{NN}_{x} \cap \text{NN}_{z}|~ /~ |\text{NN}_{x}|$,
which indicates how well local neighborhoods are preserved in the latent space of different models.



Figure~\ref{fig:recall_yelp} 
shows that DAAE consistently gives the highest recall, about 1.5$\sim$2 times that of AAE, implying that input perturbations have a substantial effect on shaping the latent space geometry. 
Table~\ref{tab:examples_nn_yelp}
presents the five nearest neighbors found by AAE and DAAE in their latent space for example test set sentences. The AAE sometimes encodes entirely unrelated sentences close together, while the latent space geometry of the DAAE is structured based on key words such as ``attentive'' and ``personable'', and tends to group sentences with similar semantics close together.
These findings are consistent with our previous conclusions in Section~\ref{sec:geometry}.

\subsection{Generation-Reconstruction Trade-off}\label{sec:gen_rec_tradeoff}

\begin{figure}[t]
\centering
\includegraphics[width=0.45\textwidth]{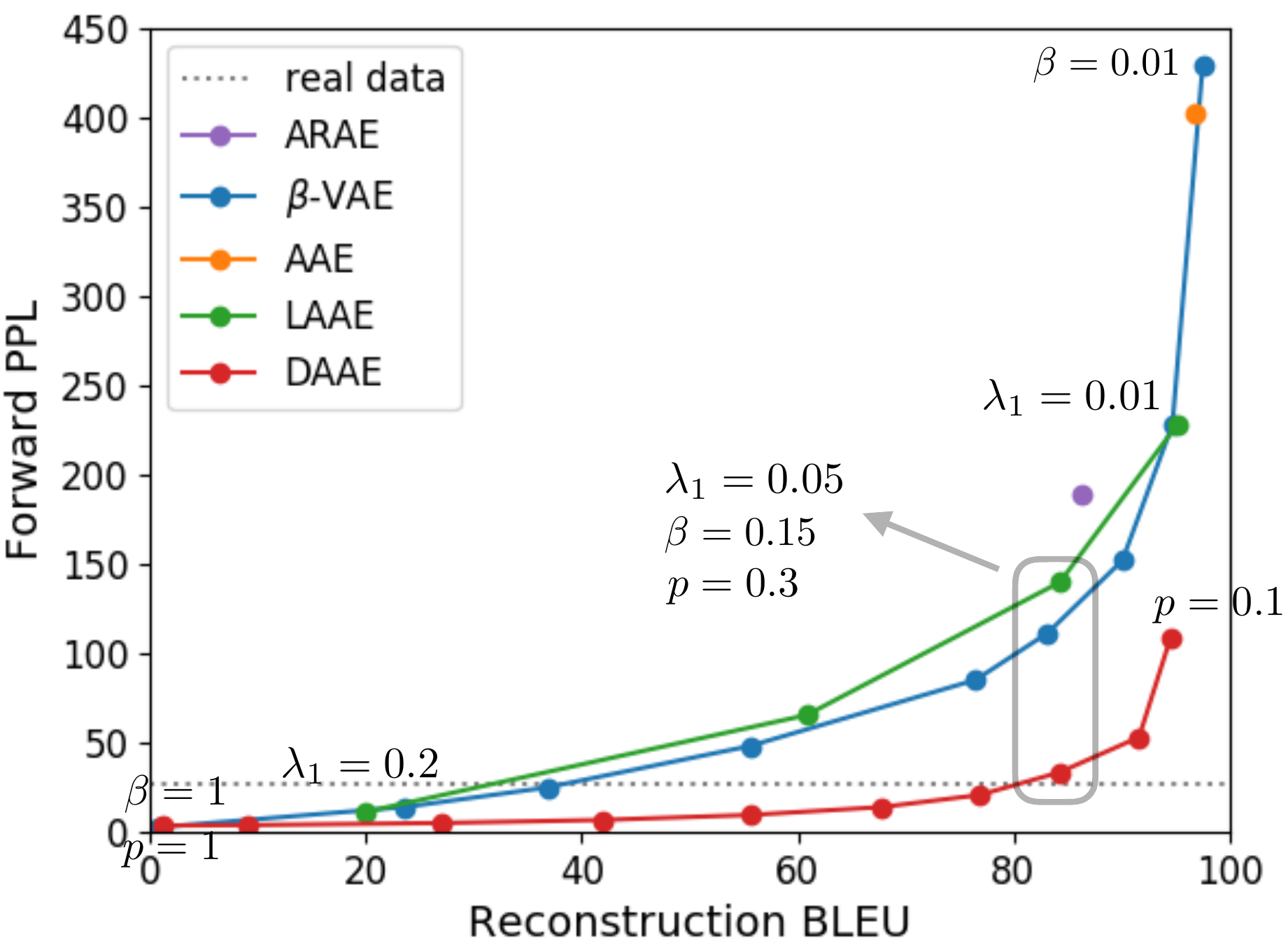}
\includegraphics[width=0.45\textwidth]{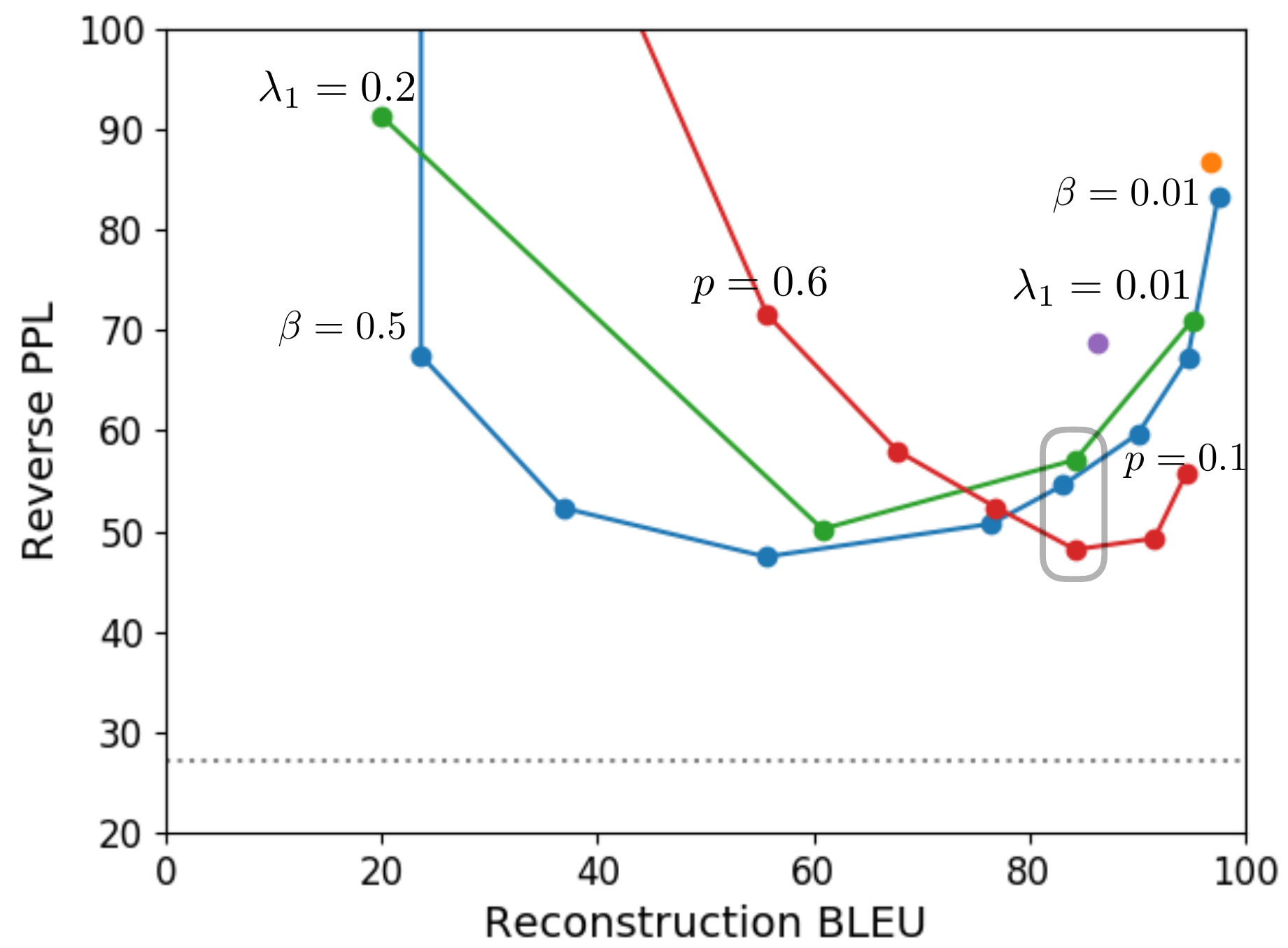}
\caption{\small Generation-reconstruction trade-off of various text autoencoders on the Yelp dataset.
The ``real data'' line marks the PPL of a language model trained and evaluated on real data.
We strive to approach the lower right corner with both high BLEU and low PPL.
The grey box identifies hyperparameters we finalize for respective models.
Points of severe collapse (Reverse PPL > 200) are removed from the lower panel.
}
 \label{fig:gen_rec_yelp}
\end{figure}

In this section, we evaluate various generative autoencoders in terms of both generation quality and reconstruction accuracy. A strong model should not only generate high quality sentences, but also learn useful latent representations that capture significant data content.
Recent work on text autoencoders has found an inherent tension between these aims~\citep{bowman2015generating}, yet only when both goals are met can we successfully manipulate sentences by modifying their latent representation (in order to produce valid output sentences that retain the semantics of 
 the input).

We compute the BLEU score~\citep{papineni2002bleu} between input and reconstructed sentences to measure reconstruction accuracy, and compute Forward/Reverse PPL to measure sentence generation quality~\citep{zhao2018adversarially,cifka2018eval}.\footnote{
While some use importance sampling estimates of data likelihood to evaluate VAEs~\citep{he2019lagging}, adopting the encoder as a proposal density is not suited for AAE variants, as they are optimized based on Wasserstein distances rather than likelihoods and lack closed-form posteriors.} 
Forward PPL is the perplexity of a language model trained on real data and evaluated on generated data. It measures the fluency of the generated text, but cannot detect the collapsed case where the model repeatedly generates a few common sentences. Reverse PPL is the perplexity of a language model trained on generated data and evaluated on real data. It takes into account both the fluency and diversity of the generated text. If a model generates only a few common sentences, a language model trained on it will exhibit poor PPL on real data.

We thoroughly investigate the performance of different models and the trade-off between generation and reconstruction.
Figure~\ref{fig:gen_rec_yelp} plots reconstruction BLEU (higher is better) vs. Forward/Reverse PPL (lower is better). The lower right corner indicates an ideal situation where good reconstruction accuracy and generation quality are both achieved.
For models with tunable hyperparameters, we sweep the full spectrum of their  generation-reconstruction trade-off by varying the KL coefficient $\beta$ of $\beta$-VAE, the log-variance $L_1$ penalty $\lambda_1$ of LAAE, and the word drop probability $p$ of DAAE.\footnote{We also studied the VAE with word dropout on the decoder side proposed by \citet{bowman2015generating}, but found that it exhibited poor reconstruction over all settings of the dropout parameter (best BLEU $= 12.8$ with dropout rate $= 0.7$). Thus this model is omitted from our other analyses.}

In the upper panel, we observe that a standard VAE ($\beta=1$) completely collapses and ignores the latent variable $z$, resulting in a reconstruction BLEU close to 0. At the other extreme, AAE can achieve near-perfect reconstruction, but its latent space is highly non-smooth and generated sentences are of poor quality, indicated by its large Forward PPL. Decreasing $\beta$ in VAE or introducing latent noises in AAE provides the model with a similar trade-off curve between reconstruction and generation.
We note that ARAE falls on or above their curves, revealing that it does not fare better than these methods (\citet{cifka2018eval} also reported similar findings).
Our proposed DAAE provides a trade-off curve that is strictly superior to other models.
With discrete $x$ and a complex encoder, the Gaussian perturbations added to the latent space by $\beta$-VAE and LAAE are not directly related to how the inputs are encoded. In contrast, input perturbations added by DAAE can constrain the encoder to maintain coherence between neighboring inputs in an end-to-end fashion and help learn smoother latent space.


The lower panel in Figure~\ref{fig:gen_rec_yelp}  illustrates that Reverse PPL first drops and then rises as we increase the degree of regularization/perturbation. This is because when $z$ encodes little information, generations from prior-sampled $z$ lack enough diversity to cover the real data.
Again, DAAE outperforms other models that tend to have higher PPL and lower BLEU. 

Based on these results, we set $\beta=0.15$ for $\beta$-VAE, $\lambda_1=0.05$ for LAAE, and $p=0.3$ for DAAE in the neighborhood preservation and text manipulation experiments, to ensure they have strong  reconstruction abilities and encode enough information about data.



\subsection{Style Transfer via Latent Vector Arithmetic}\label{sec:style_transfer}

\citet{mikolov2013linguistic} previously discovered that
word embeddings from unsupervised learning can capture linguistic relationships via simple arithmetic. A canonical example is the embedding arithmetic ``King'' - ``Man'' + ``Woman''
 $\approx$ ``Queen''.
Here, we use the Yelp dataset with tense and sentiment as two example attributes~\citep{hu2017toward,shen2017style} to investigate whether analogous structure emerges in the latent space of our sentence-level models.

\paragraph{Tense}
We use the Stanford Parser
to extract the main verb of a sentence and determine the sentence tense based on its part-of-speech tag. We compute a single ``tense vector'' by averaging the latent code $z$ separately for 100 (non-parallel) past tense sentences and present tense sentences in the development set, and then calculating the difference between the two. Given a sentence from the test set, we attempt to change its tense from past to present or from present to past through simple addition/subtraction of the tense vector. More precisely, a source sentence $x$ is first is encoded to $z = E(x)$, and then the tense-modified sentence is produced via $G(z \pm v)$, where $v \in \mathbb{R}^d$ denotes the fixed tense vector.

To quantitatively compare different models, we compute their tense transfer accuracy as measured by the parser, output BLEU with the input sentence, and output (forward) PPL evaluated by a language model. DAAE achieves the highest accuracy, lowest PPL, and relatively high BLEU (Table~\ref{tab:eval_tense}, Above), indicating that the output sentences produced by our model are more likely to be of high quality and of the proper tense, while remaining similar to the source sentence.
A human evaluation on 200 test sentences (100 past and 100 present, details in Appendix~\ref{sec:humaneval}) suggests that DAAE outperforms $\beta$-VAE twice as often as it is outperformed, 
and it successfully inverts tense for  $(48+26)/(200-34)=44.6\%$ of sentences, 13.8\% more than $\beta$-VAE (Table~\ref{tab:eval_tense}, Below).  
Table~\ref{tab:examples_tense} shows the results of adding or subtracting this fixed latent vector offset under different models. 
DAAE properly changes ``enjoy'' to ``enjoyed'' or the subjunctive mood to declarative mood. 
Other baselines either fail to alter the tense, or undesirably  change the semantic meaning of the source sentence (e.g.\ ``enjoy'' to ``made'').

\begin{table}[t]
\centering
\begin{tabular}{lccc}
\toprule
\textbf{Model} & \textbf{ACC} & \textbf{BLEU} & \textbf{PPL} \\
\midrule
ARAE &  17.2 &	55.7 &	59.1\\
$\beta$-VAE	&  49.0 &	43.5 &	44.4 \\
AAE &  9.7 & \bf 82.2 &	37.4\\
LAAE &  43.6 &	37.5 &	55.8\\
DAAE &  \bf 50.3 &	54.3 & \bf	32.0 \\
\bottomrule
\end{tabular}

\begin{tabular}{cccccc}
\\
\toprule
$\beta$-VAE is better: 25
\ \ \ DAAE is better: \textbf{48}
\\
\midrule
both good: 26 \hspace*{10pt} both bad: 67 \hspace*{10pt} n/a: 34 
\\ 
\bottomrule
\end{tabular}

\caption{\small
Above: automatic evaluations of vector arithmetic for tense inversion.
Below: human evaluation statistics of our model vs.\ the closest baseline $\beta$-VAE.
} \label{tab:eval_tense}
\end{table}

\paragraph{Sentiment}
Following the same procedure to alter tense, we compute a ``sentiment vector'' $v$ from 100 negative and positive sentences and use it to change the sentiment of test sentences.
Table~\ref{tab:res_sentiment} reports automatic evaluations, and Table~\ref{tab:examples_sentiment} shows examples generated by AAE and DAAE.
Scaling $\pm v$ to $\pm 1.5v$ and $\pm 2v$, we find that resulting sentences get more and more positive/negative.
However, the PPL for AAE increases rapidly with the scaling factor, indicating that the sentences become unnatural when their encodings have a large offset. DAAE enjoys a much smoother latent space than AAE. 
At this challenging \emph{zero-shot} setting where
no style labels are provided during training,
DAAE with $\pm 1.5v$ is able to transfer sentiment fairly well.

\begin{table}[t]
\centering  
\begin{tabular}{ll|ccc} 
\hline 
\textbf{Model} &  & \textbf{ACC} & \textbf{BLEU} & \textbf{PPL} \TBstrut  \\
\hline 
\multicolumn{2}{l|}{\citet{shen2017style}} & 81.7 & 12.4 & 38.4 \TBstrut \\
\hline
\multirow{3}{*}{AAE} & $\pm v$ & 7.2 &	86.0 & 33.7 \\
 & $\pm 1.5v$ & 25.1 &	59.6 &	59.5 \\
 & $\pm 2v$ & 57.5 & 27.4 & 139.8 \\
\hline
\multirow{3}{*}{DAAE} & $\pm v$ & 36.2 &	40.9 &	40.0 \\
 & $\pm 1.5v$ & 73.6 & 18.2 & 54.1 \\
 & $\pm 2v$ & 91.8 & 7.3 & 61.8 \Bstrut 
 \\
\toprule
\end{tabular}  
\caption{\small Automatic evaluations of vector arithmetic for sentiment transfer. Accuracy (ACC) is measured by a sentiment classifier. The model of \citet{shen2017style} is specifically trained for sentiment transfer with labeled data, while our text autoencoders are not.}\label{tab:res_sentiment}
\end{table}  

\subsection{Sentence Interpolation via Latent Space Traversal}\label{sec:sentence_interpolation}

We also
study sentence interpolation by traversing  the latent space of text autoencoders. Given two input sentences, we encode them to $z_1,z_2$ and decode from  $tz_1+(1-t)z_2~(0\le t\le 1)$. 
Ideally, this should produce fluent sentences with gradual semantic change. 
Table~\ref{tab:examples_lerp_yelp}
shows two examples from the Yelp dataset, where it is clear that DAAE produces more coherent and natural interpolations than AAE. 
Table~\ref{tab:examples_lerp_yahoo} in the appendix shows two difficult examples from the Yahoo dataset, where we interpolate between dissimilar sentences. While it is challenging to generate semantically correct sentences in these cases, the latent space of our model exhibits continuity on topic and syntactic structure.

\begin{table*}[p] 
    \def\arraystretch{1.1}\setlength{\tabcolsep}{3pt}
    \fontsize{8.3}{10}\selectfont
    \centering
    \begin{tabular}{lll}
           \toprule
           \textbf{Input} & \textbf{i enjoy hanging out in their hookah lounge .} & \textbf{had they informed me of the charge i would n't have waited .} \\
           ARAE & i enjoy hanging out in their 25th lounge . & amazing egg of the may i actually !\\
           $\beta$-VAE & i made up out in the backyard springs salad . & had they help me of the charge i would n't have waited .\\
           AAE & i enjoy hanging out in their brooklyn lounge . & have they informed me of the charge i would n't have waited .\\
           LAAE & i enjoy hanging out in the customized and play . & they are girl ( the number so i would n't be forever .\\
           DAAE & i enjoyed hanging out in their hookah lounge . & they have informed me of the charge i have n't waited .\\
           \toprule  
    \end{tabular}
    \caption{\small Examples of vector arithmetic for tense inversion.}
    \label{tab:examples_tense}
\end{table*}

\begin{table*}[p]
    \def\arraystretch{1.1}\setlength{\tabcolsep}{3pt}
    \fontsize{8.3}{10}\selectfont
    \centering
    \begin{tabular}{lll}
            \toprule
            & AAE & DAAE\\
            \midrule
            \textbf{Input} & \textbf{the food is entirely tasteless and slimy .} & \textbf{the food is entirely tasteless and slimy .}\\
            $+v$ & the food is entirely tasteless and slimy . & the food is tremendous and fresh .\\
            $+1.5v$ & the food is entirely tasteless and slimy . & the food is sensational and fresh .\\
            $+2v$ & the food is entirely and beef . & the food is gigantic . \\
            \midrule
            \textbf{Input} & \textbf{i really love the authentic food and will come back again .} & \textbf{i really love the authentic food and will come back again .}\\
            $-v$ & i really love the authentic food and will come back again . & i really love the authentic food and will never come back again .\\
            $-1.5v$ & i really but the authentic food and will come back again . & i really do not like the food and will never come back again .\\
            $-2v$ & i really but the worst food but will never come back again . & i really did not believe the pretentious service and will never go back .\\
           \toprule
    \end{tabular}
    \caption{\small Examples of vector arithmetic for sentiment transfer.}
    \label{tab:examples_sentiment}
\end{table*}

\begin{table*}[p]
    \def\arraystretch{1.1}\setlength{\tabcolsep}{3pt}
    \fontsize{8.3}{10}\selectfont
    \centering
    \begin{tabular}{lll}
            \toprule
            \textbf{Input 1} & \textbf{it 's so much better than the other chinese food places in this area .} & \textbf{fried dumplings are a must .}\\
            \textbf{Input 2} & \textbf{better than other places .} & \textbf{the fried dumplings are a must if you ever visit this place .}\\
            \midrule
            AAE & it 's so much better than the other chinese food places in this area . & fried dumplings are a must .\\
            & it 's so much better than the other food places in this area . & fried dumplings are a must .\\
            & better , much better . & the dumplings are a must if you worst .\\
            & better than other places . & the fried dumplings are a must if you ever this place .\\
            & better than other places . & the fried dumplings are a must if you ever visit this place .\\
            \midrule
            DAAE & it 's so much better than the other chinese food places in this area . & fried dumplings are a must .\\
            & it 's much better than the other chinese places in this area . & fried dumplings are a must visit .\\
            & better than the other chinese places in this area . & fried dumplings are a must in this place .\\
            & better than the other places in charlotte . & the fried dumplings are a must we ever visit this .\\
            & better than other places . & the fried dumplings are a must if we ever visit this place .\\
           \toprule
    \end{tabular}
    \caption{\small Interpolations between two input sentences generated by AAE and our model on the Yelp dataset.}
    \label{tab:examples_lerp_yelp}
\end{table*}

%% file: conclusion.tex
\section{Conclusion}

This paper provided a thorough analysis of the latent space representations of text autoencoders. 
We showed that simply minimizing the divergence between data and model distributions cannot ensure that the data structure is preserved in the latent space, but straightforward denoising techniques can greatly improve text representations. 
We offered a theoretical explanation for these phenomena by analyzing 
the latent space geometry arisen from input perturbations. 
Our results may also help explain the success of BERT~\citep{devlin2018bert}, whose masked language modeling objective is similar to a denoising autoencoder. 

Our proposed DAAE substantially outperforms other text autoencoders in both generation and reconstruction capabilities,
and demonstrates the potential for various text manipulations via simple latent vector arithmetic.
Future work may explore more sophisticated perturbation strategies besides the basic random word deletion, or investigate what additional properties of latent space geometry help  provide finer control over text generation with autoencoders.
Beyond our theory which considered autoencoders that have perfectly optimized their objectives, we hope to see additional analyses in this area that account for the initialization/learning-process and analyze other types of autoencoders.

%% file: appendix.tex
\renewcommand{\theHsection}{A\arabic{section}}
\renewcommand{\theHtable}{A\arabic{table}}
\renewcommand{\theHfigure}{A\arabic{figure}}

\appendix

\renewcommand\thefigure{\thesection.\arabic{figure}}    
\setcounter{figure}{0}

\renewcommand\thetable{\thesection.\arabic{table}}    
\setcounter{table}{0}

\thispagestyle{empty}
\twocolumn[
\vspace*{-3.5mm}
\begin{center}
\toptitlebar
{\Large \bf
\papertitle \\
\vspace{8pt}
Supplementary Material
}
\bottomtitlebar
\end{center}
]

\section{Wasserstein Distance}\label{sec:wasserstein}
The AAE objective can be connected to a relaxed form of the Wasserstein distance between model and data distributions~\citep{tolstikhin2017wasserstein}.
Specifically, for cost function $c(\cdot,\cdot):\mathcal X\times\mathcal X\rightarrow \mathbb R$ and  deterministic decoder mapping $G:\mathcal Z\rightarrow\mathcal X$, it holds that:  
\begin{align}\label{eq:wae}
& \inf_{\Gamma\in \mathcal P(x\sim p_{\text{data}},y\sim p_G)} \mathbb E_{(x,y)\sim\Gamma}[c(x,y)] \nonumber \\
=~ & \inf_{q(z|x):q(z)=p(z)}\mathbb E_{p_{\text{data}}(x)}\mathbb E_{q(z|x)}[c(x,G(z))]
\end{align}
where the minimization over couplings $\Gamma$ with marginals $p_{\text{data}}$ and $p_G$ can be replaced with  minimization over conditional distributions $q(z|x)$ whose marginal $q(z)=\mathbb E_{p_{\text{data}}(x)}[q(z|x)]$ matches the latent prior distribution $p(z)$.
Relaxing this marginal constraint via a divergence penalty $D(q(z)\|p(z))$
 estimated by adversarial training, one recovers the AAE objective (Eq.~\ref{eq:aae}).
In particular, AAE on discrete $x$ with the cross-entropy loss is minimizing an upper bound of the total variation distance between $p_{\text{data}}$ and $p_G$, with $c$ chosen as the indicator cost function~\citep{zhao2018adversarially}.

Our model is optimizing over conditional distributions $q(z|x)$ of the form (\ref{eq:perturbedposterior}), a subset of all possible conditional distributions. Thus, after introducing input perturbations, our method is still minimizing an upper bound of the Wasserstein distance between $p_{\text{data}}$ and $p_G$ described in (\ref{eq:wae}).

\section{Toy Experiments With Latent Dimension 5}\label{sec:toy_experiment_dim5}
Here we repeat our toy experiment with clustered data, this time using a larger latent space with 5 dimensions.

\begin{figure}[h]
\centering
\includegraphics[width=0.48\textwidth]{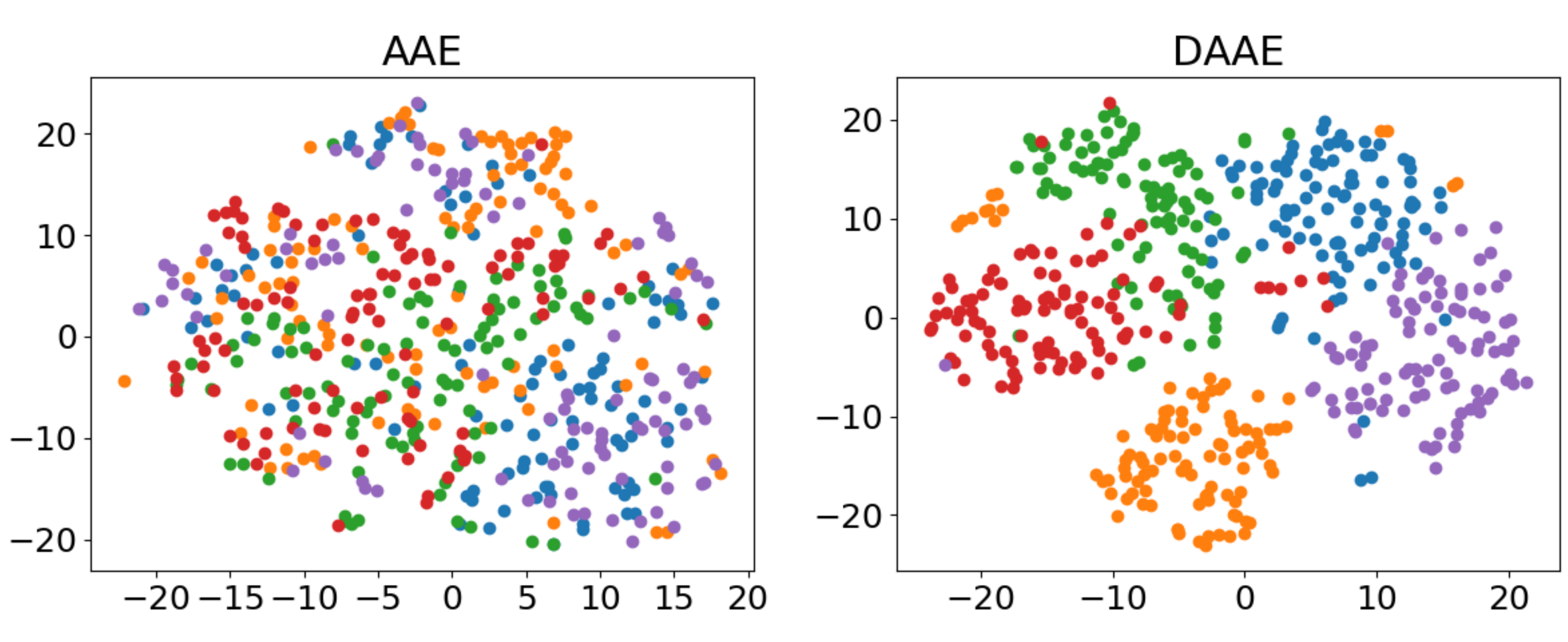}
\vspace{-9pt}
\caption{\small 
$t$-SNE visualization of 5-D latent representations learned by AAE and DAAE when mapping clustered sequences in 
$\mathcal X=\{0, 1\}^{50}$ to $\mathcal Z=\mathbb R^5$. The training data stem from 5 underlying clusters, with 100 sequences sampled from each (colored accordingly by cluster identity). 
} 
\label{fig:toy_exp_z5}
\end{figure}

\section{Proof of Theorem~\ref{thm:badaae}} \label{sec:badaaeproof}

\badaae*

\begin{proof}
Consider two encoder matchings $x_i$ to $z_{\alpha(i)}$ and $x_i$ to $z_{\beta(i)}$, where both $\alpha$ and $\beta$ are permutations of the indices $\{1,\dots,n\}$. 
Suppose $G_\alpha$ is the optimal decoder model for the first matching (with permutations $\alpha$). 
This implies 
$$
p_{G_\alpha} = \argmax_{G \in \mathcal G_L} \sum_{i=1}^n {\log p_G(x_i | z_{\alpha(i)}) } 
$$
Now let $p_{G_\beta}(x_i|z_j)=p_{G_\alpha}(x_{\beta\alpha^{-1}(i)}|z_j),\forall i, j$. Then $G_\beta$ can achieve exactly the same log-likelihood objective value for matching $\beta$ as $G_\alpha$ for matching $\alpha$, while still respecting the Lipschitz constraint.
\end{proof}

\section{Proof of Theorem~\ref{thm:fourpoints}} \label{sec:fourpointproof}

\fourpoints*

\begin{proof}

Let $[n]$ denote $\{1,\dots, n\}$, and assume without loss of generality that the encoder $E$ maps each $x_i$ to $z_i$. We also define $A = \{1,2\}, B = \{3, 4\}$ as the two $x$-pairs that lie close together. 
For our choice of $C(x)$,
the training objective to be maximized is:
\begin{align*}
 & \sum_{i,j \in A} \log p_G(x_i | E(x_j)) + \sum_{k,\ell \in B} \log p_G(x_k | E(x_\ell))
\\
= & \sum_{i,j \in A} \log p_G(x_i | z_j) + \sum_{k,\ell \in B} \log p_G(x_k | z_\ell )  \numberthis \label{eq:fourpointobj}
\end{align*}

The remainder of our proof is split into two cases:

\textbf{Case 1.} $||z_j - z_\ell || > \zeta$ for $j \in A, \ell \in B$ 

\textbf{Case 2.} $||z_{j}- z_{\ell} ||  < \delta $ for $j \in A, \ell \in B$ 

Under Case 1, $x$ points that lie far apart also have $z$ encodings that remain far apart. Under Case 2, $x$ points that lie far apart have $z$ encodings that lie close together.
We complete the proof by showing that the achievable objective value in Case 2 is strictly worse than in Case 1, and thus an optimal encoder/decoder pair
would avoid the $x,z$ matching that leads to Case 2.

In Case 1 where $||z_j - z_\ell || > \zeta$ for all $j \in A, \ell \in B$, we can lower bound the training objective (\ref{eq:fourpointobj}) by choosing:
\begin{equation}
p_G(x_i | z_j) = 
\begin{cases}
(1-\gamma)/2 \ \ & \text{ if  $i,j \in A$ or $i,j \in B$ } \\
\gamma/2 \ \ & \text{ otherwise}
\end{cases}
\label{eq:optimalp}
\end{equation}
with $\gamma = \sigma(-L \zeta) \in (0,\frac{1}{2})$, where $\sigma(\cdot)$ denotes the sigmoid function.
Note that this ensures $\displaystyle \sum_{i \in [4]} p_G(x_i | z_j) = 1$ for each $j \in [4]$, and does not violate the Lipschitz condition from Assumption~\ref{asm:Lipschitz} since:
\begin{align*}
 | \log p_G(x_i | z_j) - \log p_G(x_i | z_\ell) |
\end{align*}
\vspace{-20pt}
\begin{align*}
\begin{cases}
 ~= 0 & \ \text{ if $j, \ell \in A$ or $j, \ell \in B$}
\\
 ~\le   \log \left( (1-\gamma)/\gamma \right) & \ \text { otherwise}
\end{cases}
\end{align*}

and thus remains $\le L || z_j - z_\ell ||$ when $\gamma = \sigma(-L \zeta)  \ge  \sigma(-L||z_j - z_\ell||) = 1/ [1 + \exp(L||z_j - z_\ell||)]
$.

Plugging the $p_G(x| z)$ assignment from (\ref{eq:optimalp}) into (\ref{eq:fourpointobj}), we see that an optimal decoder can obtain training objective value $ \ge 8 \log \left[\sigma(L \zeta)/2 \right]$ 
in Case 1 where  $||z_j - z_\ell || > \zeta, \ \forall j \in A, \ell \in B$.

\bigskip 

Next, we consider the alternative case where $||z_{j}- z_{\ell} ||  < \delta $ for $j \in A, \ell \in B$.

For $i,j \in A$ and for all $\ell \in B$, we have:
\begin{align*}
    \log p_G(x_i | z_{j}) \le~ & \log p_G(x_i | z_{\ell}) + L||z_{j} - z_{\ell}|| \tag*{(by Assumption~\ref{asm:Lipschitz})} 
\\
\le~ & \log p_G(x_i | z_{\ell}) + L\delta  \ \
\\
\le~ &  L\delta + \log \left[ 1 - \sum_{k \in B} p_G(x_k | z_{\ell}) \right] 
\tag*{(since $\sum_{k} p_G(x_k | z_{\ell}) \le 1$)}
\end{align*}

Continuing from (\ref{eq:fourpointobj}), the overall training objective in this case is thus:

\begin{align*}
& \sum_{i,j \in A} \log p_G(x_i | z_j) + \sum_{k,\ell \in B} \log p_G(x_k | z_\ell )
\\
\le~ & 4L\delta + \sum_{i,j \in A} \min_{\ell \in B}  \log \left[ 1 - \sum_{k \in B}  p_G(x_k | z_{\ell}) \right] \\
 & + \sum_{k,\ell \in B} \log p_G(x_k | z_\ell ) 
\\
\le~ & 4L\delta + \sum_{\ell \in B} \left[ 
 2 \log \left( 1 - \sum_{k \in B}  p_G(x_k | z_{\ell}) \right) \right. \\
 & \left. + \sum_{k \in B} \log p_G(x_k | z_\ell ) \right]
\\
\le~ & 4L\delta - 12 \log 2
\end{align*}
using the fact that the optimal decoder for the bound in this case is: $p_G(x_k | z_{\ell}) = 1/4$ for all $k,\ell \in B$.

Finally, plugging our range for $\delta$ stated in the Theorem~\ref{thm:fourpoints}, it shows that the best achievable objective value in Case 2 is strictly worse than the objective value achievable in Case 1.  Thus, the optimal encoder/decoder pair under the AAE with perturbed $x$ will always prefer the matching between $\{x_1,\dots, x_4\}$ and $\{z_1,\dots, z_4\}$ that ensures nearby $x_i$ are encoded to nearby $z_i$ (corresponding to Case 1).
\end{proof}

\section{Proof of Theorem~\ref{thm:clusteredx}}\label{sec:proof}
\clusteredx*
\begin{proof}
Without loss of generality, let $E(x_i)=z_i$ for notational convenience.
We consider what is the optimal decoder probability assignment $p_G(x_i|z_j)$ under the Lipschitz constraint~\ref{asm:Lipschitz}.

The objective of the AAE with perturbed $x$ is to maximize:
\begin{align*}
& \frac{1}{n}\sum_i\sum_j p_C(x_j|x_i) \log p_G(x_i|E(x_j)) \\
=~ & \frac{1}{nK} \sum_j \sum_{i:S_i=S_j} \log p_G(x_i|z_j)
\end{align*}
We first show that the optimal $p_G(\cdot|\cdot)$ will satisfy that the same probability is assigned within a cluster, i.e. $p(x_i|z_j)=p(x_k|z_j)$ for all $i,k$ s.t. $S_i=S_k$. If not, let $P_{sj}=\sum_{i:S_i=s} p_G(x_i|z_j)$, and we reassign $p_{G'}(x_i|z_j)=P_{S_ij}/K$. Then $G'$ still conforms to the Lipschitz constraint if $G$ meets it, and $G'$ will have a larger target value than $G$.

Now let us define $P_j=\sum_{i:S_i=S_j} p_G(x_i|z_j)=K\cdot p_G(x_j|z_j)$ $(0\le P_j\le 1)$.
The objective becomes:
\begin{align*}
    & \max_{p_G} \frac{1}{nK} \sum_j \sum_{i:S_i=S_j} \log p_G(x_i|z_j)\\
    =~ & \max_{p_G}\frac{1}{n} \sum_j \log p_G(x_j|z_j)\\
    =~ & \max_{p_G}\frac{1}{n} \sum_j \log P_j - \log K\\
    =~ & \max_{p_G}\frac{1}{2n^2} \sum_i\sum_j (\log P_i + \log P_j) - \log K\\
    \le~ & \frac{1}{2n^2} \sum_i\sum_j\max_{p_G} (\log P_i + \log P_j) - \log K
\end{align*}

Consider each term ${\max_{p_G} (\log P_i + \log P_j)}$: when $S_i=S_j$, this term can achieve the maximum value $0$ by assigning ${P_i=P_j=1}$; when ${S_i\ne S_j}$, the Lipschitz constraint ensures that:
\begin{align*}
    \log (1-P_i) \ge \log P_j - L\|z_i-z_j\|\\
    \log (1-P_j) \ge \log P_i - L\|z_i-z_j\|
\end{align*}
Therefore:
\begin{align*}
    \log P_i + \log P_j \le 2\log\sigma(L\|z_i-z_j\|)
\end{align*}
Overall, we thus have:
\begin{align*}
    & \max_{p_G} \frac{1}{nK} \sum_j \sum_{i:S_i=S_j} \log p_G(x_i|z_j) \nonumber \\
    \le~ & \frac{1}{n^2} \sum_{i,j:S_i\ne S_j} \log\sigma(L\|z_i-z_j\|) - \log K
\end{align*} \qedhere
\end{proof}

\section{Experimental Details}\label{sec:expdetail}
We use the same architecture to implement all models with different objectives. The encoder $E$, generator $G$, and the language model used to compute Forward/Reverse PPL are one-layer LSTMs with hidden dimension 1024 and word embedding dimension 512. The last hidden state of the encoder is projected into 128/256 dimensions to produce the latent code $z$ for Yelp/Yahoo datasets respectively, which is then projected and added with input word embeddings fed to the generator. The discriminator $D$ is an MLP with one hidden layer of size 512.
$\lambda$ of AAE based models is set to 10 to ensure the latent codes are indistinguishable from the prior.
All models are trained via the Adam optimizer~\citep{kingma2014adam} with learning rate 0.0005, $\beta_1=0.5$, $\beta_2=0.999$.
At test time, encoder-side perturbations are disabled, and we use greedy decoding to generate $x$ from $z$.

\section{Human Evaluation}\label{sec:humaneval}
For the tense transfer experiment, the human annotator is presented with a source sentence and two outputs (one from each approach, presented in random order) and asked to judge which one successfully changes the tense while being faithful to the source, or whether both are good/bad, or if the input is not suitable to have its tense inverted.
We collect labels from two human annotators and if they disagree, we further solicit a label from the third annotator.

\clearpage

\onecolumn

\section{Neighborhood Preservation}


Here we include non-generative models AE, DAE and repeat the neighborhood preservation analysis in Section~\ref{sec:neighborhood}. We find that an untrained RNN encoder from random initialization has a good recall rate, and we suspect that SGD training of vanilla AE towards only the reconstruction loss will not overturn this initial bias. Note that denoising still improves neighborhood preservation in this case. Also note that DAAE has the highest recall rate among all generative models that have a latent prior imposed.

\begin{figure*}[!htbp]
\centering
\includegraphics[width=0.99\textwidth]{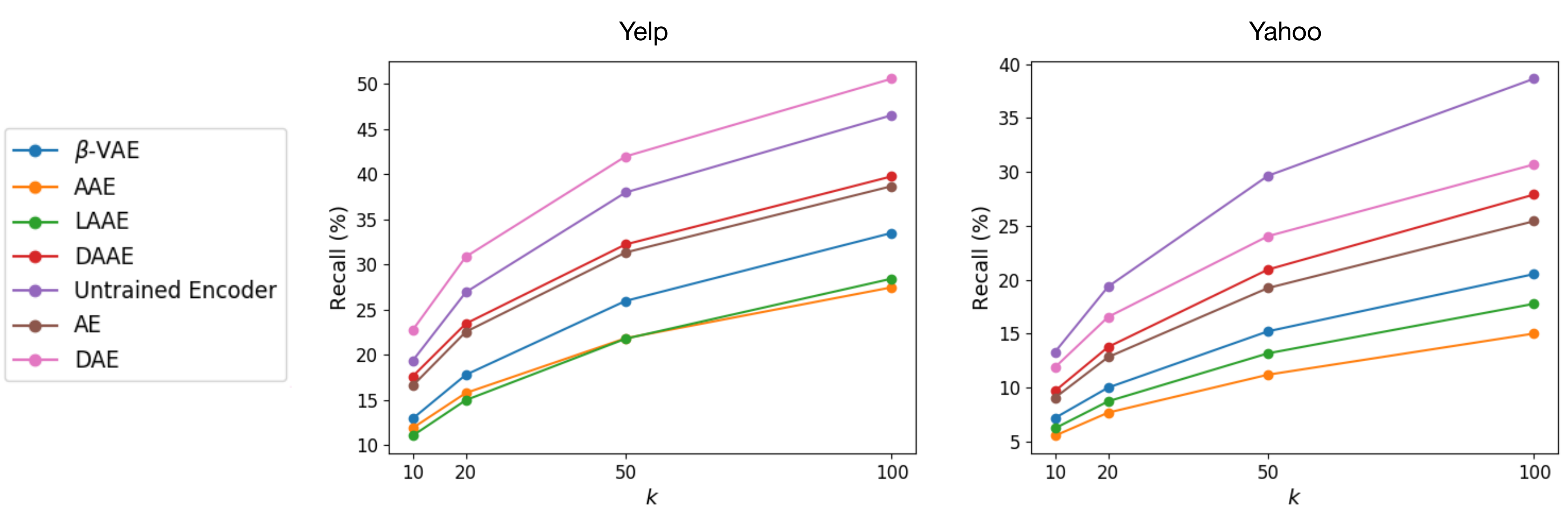}
\caption{\small Recall rate of 10 nearest neighbors in the sentence space retrieved by $k$ nearest neighbors in the latent space of different autoencoders on the Yelp and Yahoo datasets.
}
\label{fig:recall_all}
\end{figure*}

\vspace{50pt}

\begin{table*}[ht]
    \def\arraystretch{1.1}\setlength{\tabcolsep}{3pt}
    \scriptsize
    \centering
    \begin{tabular}{ll}
            \toprule
            \textbf{Source} & \textbf{how many gospels are there that were n't included in the bible ?}\\
            \\
            5-NN by AAE & there are no other gospels that were n't included in the bible . \\
            & how many permutations are there for the letters in the word \_UNK ' ? \\
            & anyone else picked up any of the \_UNK in the film ? \\
            & what 's the significance of the number 40 in the bible ? \\
            & how many pieces of ribbon were used in the \_UNK act ? \\
            \\
            5-NN by DAAE & there are no other gospels that were n't included in the bible . \\
            & how many litres of water is there in the sea ? \\
            & how many \_UNK gods are there in the classroom ? \\
            & how many pieces of ribbon were used in the \_UNK act ? \\
            & how many times have you been grounded in the last year ? \\
            \midrule
            \textbf{Source} & \textbf{how do i change colors in new yahoo mail beta ?}\\
            \\
            5-NN by AAE & how should you present yourself at a \_UNK speaking exam ?\\
            & how can i learn to be a hip hop producer ?\\
            & how can i create a \_UNK web on the internet ?\\
            & how can i change my \_UNK for female not male ?\\
            & what should you look for in buying your first cello ?\\
            \\
            5-NN by DAAE & how do i change that back to english ?\\
            & is it possible to \_UNK a yahoo account ?\\
            & how do i change my yahoo toolbar options ?\\
            & what should you look for in buying your first cello ?\\
            & who do you think should go number one in the baseball fantasy draft , pujols or \_UNK ?\\
           \toprule
    \end{tabular}
    \caption{\small Examples of nearest neighbors in the latent Euclidean space of AAE and DAAE on Yahoo dataset.}
    \label{tab:examples_nn_yahoo}
\end{table*}

\clearpage

\section{Generation-Reconstruction Results on the Yahoo Dataset}

In this section, we repeat the autoencoder generation-reconstruction analysis in Section~\ref{sec:gen_rec_tradeoff} on the Yahoo dataset. As on the Yelp dataset, our DAAE model provides the best trade-off.

\vspace{10pt}

\begin{figure}[htb!]
\centering
\includegraphics[width=\textwidth]{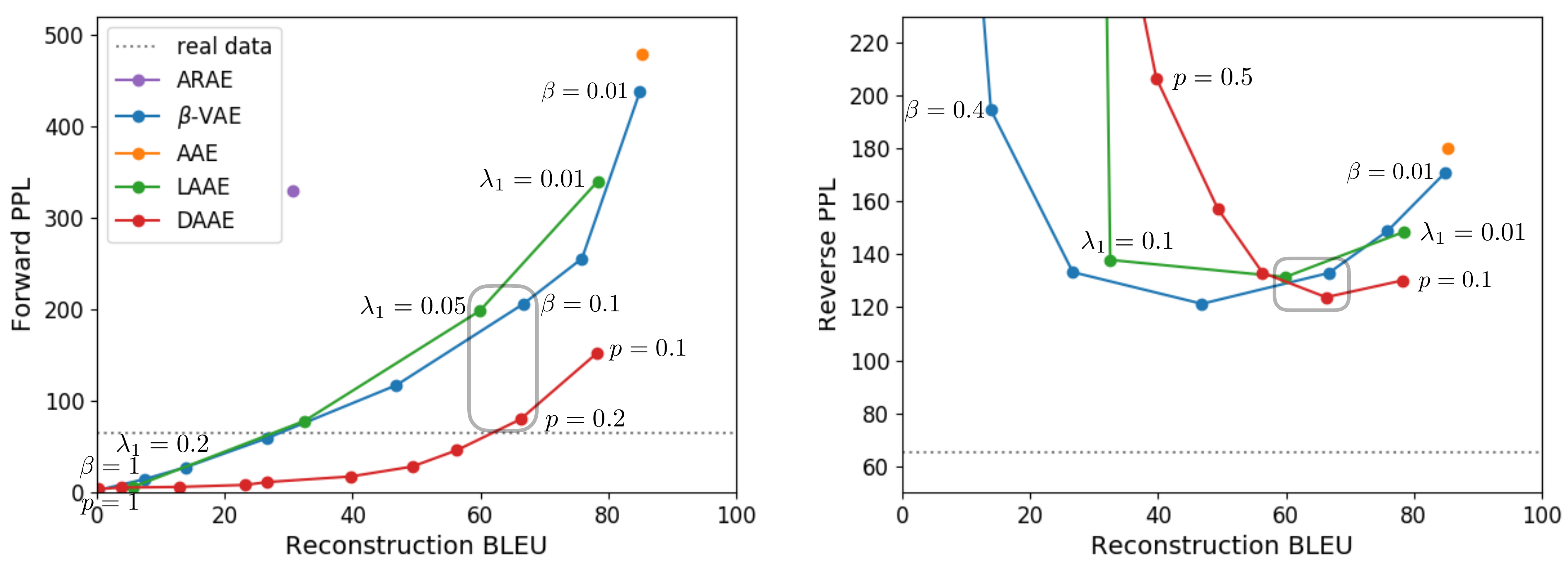}
\caption{\small Generation-reconstruction trade-off of various text autoencoders on the Yahoo dataset.
The ``real data'' line marks the PPL of a language model trained and evaluated on real data.
We strive to approach the lower right corner with both high BLEU and low PPL.
The grey box identifies hyperparameters we finalize for respective models.
Points of severe collapse (Reverse PPL > 300) are removed from the right panel.
}
 \label{fig:gen_rec_yahoo}
\end{figure}

\clearpage

\section{Additional Examples of Generated Text}

This section presents more examples of text decoded after various geometric manipulations of the latent space representations.

\vspace{10pt}

\begin{table*}[ht]
    \def\arraystretch{1.1}\setlength{\tabcolsep}{3pt}
    \scriptsize
    \centering
    \begin{tabular}{lll}
           \toprule
           \textbf{Input} & \textbf{the staff is rude and the dr. does not spend time with you .} & \textbf{slow service , the food tasted like last night 's leftovers .}\\
           ARAE & the staff is rude and the dr. does not worth two with you . & slow service , the food tasted like last night 's leftovers .\\
           $\beta$-VAE & the staff was rude and the dr. did not spend time with your attitude . & slow service , the food tastes like last place serves .\\
           AAE & the staff was rude and the dr. does not spend time with you . & slow service , the food tasted like last night 's leftovers .\\
           LAAE & the staff was rude and the dr. is even for another of her entertained . & slow service , the food , on this burger spot ! \\
           DAAE & the staff was rude and the dr. did not make time with you . & slow service , the food tastes like last night ... .\\
           \\
           \textbf{Input} & \textbf{they are the worst credit union in arizona .} & \textbf{i reported this twice and nothing was done .}\\
           ARAE & they are the worst bank credit in arizona . & i swear this twice and nothing was done .\\
           $\beta$-VAE & they were the worst credit union in my book . & i 've gone here and nothing too .\\
           AAE & they are the worst credit union in arizona . & i reported this twice and nothing was done .\\
           LAAE & they were the worst credit union in my heart . & i dislike this twice so pleasant guy .\\
           DAAE & they were the worst credit union in arizona ever . & i hate this pizza and nothing done .\\
           \bottomrule
    \end{tabular}
    \caption{\small Additional examples of vector arithmetic for tense inversion.}
    \label{tab:examples_tense_more}
\end{table*}

\vspace{50pt}

\begin{table*}[ht]
    \def\arraystretch{1.1}\setlength{\tabcolsep}{3pt}
    \fontsize{6.8}{8}\selectfont
    \centering
    \begin{tabular}{lll}
            \toprule
            & AAE & DAAE\\
            \midrule
            \textbf{Input} & \textbf{this woman was extremely rude to me .} & \textbf{this woman was extremely rude to me .}\\
            $+v$ & this woman was extremely rude to me . & this woman was extremely nice .\\
            $+1.5v$ & this woman was extremely rude to baby . & this staff was amazing .\\
            $+2v$ & this woman was extremely rude to muffins . & this staff is amazing . \\
            \midrule
            \textbf{Input} & \textbf{my boyfriend said his pizza was basic and bland also .} & \textbf{my boyfriend said his pizza was basic and bland also .}\\
            $+v$ & my boyfriend said his pizza was basic and tasty also . & my boyfriend said his pizza is also excellent .\\
            $+1.5v$ & my shared said friday pizza was basic and tasty also . & my boyfriend and pizza is excellent also .\\
            $+2v$ & my shared got pizza pasta was basic and tasty also . & my smoked pizza is excellent and also exceptional . \\
            \midrule
            \textbf{Input} & \textbf{the stew is quite inexpensive and very tasty .} & \textbf{the stew is quite inexpensive and very tasty .}\\
            $-v$ & the stew is quite inexpensive and very tasty . & the stew is quite an inexpensive and very large .\\
            $-1.5v$ & the stew is quite inexpensive and very very tasteless . & the stew is quite a bit overpriced and very fairly brown .\\
            $-2v$ & the -- was being slow - very very tasteless . & the hostess was quite impossible in an expensive and very few customers . \\
            \midrule
            \textbf{Input} & \textbf{the patrons all looked happy and relaxed .} & \textbf{the patrons all looked happy and relaxed .}\\
            $-v$ & the patrons all looked happy and relaxed . & the patrons all helped us were happy and relaxed .\\
            $-1.5v$ & the patrons all just happy and smelled . & the patrons that all seemed around and left very stressed .\\
            $-2v$ & the patrons all just happy and smelled . & the patrons actually kept us all looked long and was annoyed .\\
           \bottomrule
    \end{tabular}
    \caption{\small Additional examples of vector arithmetic for sentiment transfer.}
    \label{tab:examples_sentiment_more}
\end{table*}

\begin{table*}[ht]
    \def\arraystretch{1.1}\setlength{\tabcolsep}{3pt}
    \scriptsize
    \centering
    \begin{tabular}{ll}
            \toprule
            \textbf{Input 1} & \textbf{what language should i learn to be more competitive in today 's global culture ?}\\
            \textbf{Input 2} & \textbf{what languages do you speak ?}\\
            \\
            AAE & what language should i learn to be more competitive in today 's global culture ?\\
            & what language should i learn to be more competitive in today 's global culture ?\\
            & what language should you speak ?\\
            & what languages do you speak ?\\
            & what languages do you speak ?\\
            \\
            DAAE & what language should i learn to be more competitive in today 's global culture ?\\
            & what language should i learn to be competitive today in arabic 's culture ?\\
            & what languages do you learn to be english culture ?\\
            & what languages do you learn ?\\
            & what languages do you speak ?\\
            \midrule
            \textbf{Input 1} & \textbf{i believe angels exist .}\\
            \textbf{Input 2} & \textbf{if you were a character from a movie , who would it be and why ?}\\
            \\
            AAE & i believe angels exist .\\
            & i believe angels - there was the exist exist .\\
            & i believe in tsunami romeo or <unk> i think would it exist as the world population .\\
            & if you were a character from me in this , would we it be ( why !\\
            & if you were a character from a movie , who would it be and why ?\\
            \\
            DAAE & i believe angels exist .\\
            & i believe angels exist in the evolution .\\
            & what did <unk> worship by in <unk> universe ?\\
            & if you were your character from a bible , it will be why ?\\
            & if you were a character from a movie , who would it be and why ?\\
            
           \toprule
    \end{tabular}
    \caption{\small Interpolations between two input sentences generated by AAE and our model on the Yahoo dataset.}
    \label{tab:examples_lerp_yahoo}
\end{table*}